\documentclass[tablecaption=bottom,wcp]{jmlr} 
\usepackage[utf8]{inputenc}
\usepackage{graphicx}
\usepackage[T1]{fontenc}
\usepackage{bm}
\usepackage{amsfonts}
\usepackage{color}

\DeclareUnicodeCharacter{2212}{-}  
\usepackage{pgf}
\usepackage{tikz}
\usetikzlibrary{external}
\newcommand\inputpgf[2]{{
\let\pgfimageWithoutPath\pgfimage
\renewcommand{\pgfimage}[2][]{\pgfimageWithoutPath[##1]{#1/##2}}
\input{#1/#2}
}}

\makeatletter
\def\set@curr@file#1{\def\@curr@file{#1}} 
\makeatother
\usepackage[load-configurations=version-1]{siunitx} 

\usepackage{booktabs}
\usepackage[load-configurations=version-1]{siunitx} 

\theorembodyfont{\upshape}
\theoremheaderfont{\scshape}
\theorempostheader{:}
\theoremsep{\newline}

\usepackage{adria-math}
\usepackage{cleveref}

\newcommand{\gradat}[1]{{\nabla_\vtheta U(\vtheta #1)}}

\jmlrproceedings{AABI 2020}{3rd Symposium on Advances in Approximate Bayesian Inference, 2020}

\title{Exact Langevin Dynamics with Stochastic Gradients}

\author{%
  \Name{Adri\`a Garriga-Alonso}
  \Email{ag919@cam.ac.uk} \\
  \addr University of Cambridge, United Kingdom\\
  \Name{Vincent Fortuin}
  \Email{fortuin@inf.ethz.ch} \\
  \addr ETH Z\"urich, Switzerland}

\newcommand{\mhaccept}[2]{{\text{MH}\left( #1 \mvbar #2 \right)}}
\newcommand{\Ptarg}[1]{\tilde{\pi}\bra{#1}}
\newcommand{\Ptrans}[2]{g\bra{#1 \mvbar #2 }}
\newcommand{\Pboltzmann}[1]{\bar{P}\bra{#1}}

\begin{document}

\maketitle

\begin{abstract}

  Stochastic gradient Markov Chain Monte Carlo algorithms are popular samplers for approximate inference, but they are generally biased. We show that many recent versions of these methods (e.g. \cite{sghmc}) cannot be corrected using Metropolis-Hastings rejection sampling, because their acceptance probability is always zero. We can fix this by employing a sampler with realizable backwards trajectories, such as Gradient-Guided Monte Carlo \citep{horowitz1991}, which generalizes stochastic gradient Langevin dynamics \citep{sgld} and Hamiltonian Monte Carlo. We show that this sampler can be used with stochastic gradients, yielding nonzero acceptance probabilities, which can be computed even across multiple steps.
\end{abstract}

\section{Introduction}
\label{sec:intro}

Exact posterior inference in many Bayesian models is intractable, which has led to the development of different approximate inference schemes.
One popular class of such algorithms are Markov Chain Monte Carlo (MCMC) methods. These simulate a Markov chain, constructed so that its stationary distribution is the true posterior \citep{geyer1992practical}.
Among the MCMC methods, Hamiltonian Monte Carlo (HMC) has achieved recognition as being the ``gold standard'', mostly for its favorable convergence properties \citep{neal2012handbook}.

While HMC yields great performance, its adoption in many real applications is hindered by the fact that it requires passes over the whole data set to compute one gradient step, which is prohibitively expensive in many settings \citep{sghmc}.
Especially for large-data Bayesian models such as Bayesian neural networks (BNNs), stochastic gradient MCMC (SG-MCMC) methods are more popular, due to their favorable computational costs and their similarity to stochastic gradient descent (SGD).
Arguably, the most popular SG-MCMC method for BNNs in recent years has been stochastic gradient HMC \citep[SGHMC]{sghmc}, although its predecessor stochastic gradient Langevin dynamics \citep[SGLD]{sgld} also enjoys considerable following.

Since all of these approaches are approximate, in order to ensure that they sample from the true posterior, rejection sampling is often necessary, for instance, following the Metropolis-Hastings (MH) scheme \citep{metropolis1953equation, hastings1970monte}.
However, in this work, we draw attention to the unfortunate fact that SGHMC does not allow for MH rejection sampling, since its acceptance probability is always zero.
We then revisit gradient-guided Monte Carlo (GGMC) \citep{horowitz1991}, which can overcome this problem.
We show that GGMC provides a unified framework for SGLD and SGHMC and that it yields nonzero acceptance probabilities, which can even be computed across multiple steps.

Our contributions are as follows:
\begin{itemize}
\item We show that the SG-MCMC schemes used in recent work \citep{sghmc,wenzel20posterior} always have acceptance probability zero, because the backward trajectory is not realizable in the symplectic Euler-Maruyama integrator they employ (Sec.~\ref{sec:zero_probs}).
\item We revisit GGMC, a variant of HMC. It generalises HMC and SGLD, always has positive acceptance probability and, with correction, accepts stochastic gradients (Sec.~\ref{sec:GGMC}). The scheme was independently discovered in the statistics \citep{neal2012handbook,horowitz1991} and molecular dynamics literatures \citep{bussi-parrinello,lm-obabo}. The latter have shown that it simulates Langevin dynamics.
\item We examine an augmentation scheme to calculate the acceptance probabilities across multiple stochastic steps (\cref{sec:multi_step}, \cref{apd:multi_step}; \cite{amagold}). This allows us to use stochastic gradients, but sample from the exact posterior, by doing one MH correction using the exact likelihood after many stochastic gradient steps.
\end{itemize}

Based on our findings, we believe that work in Bayesian deep learning should discard the Euler-Maruyama scheme used by SGHMC and its variants \citep{sgnht,sghmc-complete-recipe}. Instead, to have confidence in our sampling, we should use reversible integrators \citep[this work]{nogin,adlala, amagold} and keep track of their acceptance probability.

\section{Background: Hamiltonian Monte Carlo}


We wish to sample from a target distribution $\pi\bra{\vtheta} \propto \Ptarg{\vtheta}$, which we can
only evaluate up to a normalizing constant. HMC first augments the parameter
space $\vtheta \in \mathbb{R}^d$ with a momentum $\vm\in\mathbb{R}^{d}$ of the
same dimension. It constructs a Markov chain that alternates between re-sampling the momentum and then
following an ordinary differential equation (ODE), specified by the Hamiltonian $H(\vtheta, \vm)$,
for a given amount of time.
The stationary distribution of this Markov chain (with exact simulation) is the Boltzmann distribution
\begin{equation}
  \Pboltzmann{\vtheta, \vm} \propto \exp\bra{-\frac{1}{T}H(\vtheta, \vm)} = \exp\bra{-\frac{1}{T} U(\vtheta)} \exp\bra{-\frac{1}{T}K(\vm)}.
  \label{eq:boltzmann}
\end{equation}
Here, $U(\vtheta)$ is the potential energy, which we relate to the unnormalised target distribution by $U(\vtheta) = -\log \Ptarg{\vtheta}$.
The kinetic energy $K(\vm) = \frac{1}{2}\vm^{\tp}\mM^{{-1}}\vm$  makes the momenta $\vm$ Gaussian, and thus easy to sample.
$\vM$ is a positive
definite (usually diagonal) preconditioner, the mass matrix.
 $T \in \mathbb{R}^{+}$ is a positive temperature. For HMC, $T=1$, but we will relax this constraint in the next section. By construction, the tempered target $\Ptarg{\vtheta}^{1/T}$ is a marginal of the Boltzmann distribution.
We may write the ODE's full state as $\vs = (\vtheta, \vm)$.

Unfortunately, we cannot exactly simulate the ODE for most posteriors, so we must discretise it with an integrator. This unavoidably introduces errors into the simulation, which have to be controlled if we want to guarantee that the sampling is accurate.

The standard way to control these errors is via the Metropolis-Hastings acceptance probability. For this, we define the transition kernel: $\Ptrans{\vs_*}{\vs_n} = 1$ if the approximate simulation starting at $\vs_n=(\vtheta_n, \vm_n)$ ends at $\vs_*=(\vtheta_*, \vm_*)$, otherwise 0.

Then, we accept the new sample (i.e., set $\vs_{n+1}=\vs_* $) with probability $\min\bra{1, \mhaccept{\vs_{*}}{\vs_{n}}}$, where
\begin{equation}
\mhaccept{\vs_{*}}{\vs_{n}} = \frac{\Pboltzmann{\vtheta_*, \vm_*}}{\Pboltzmann{\vtheta_n, \vm_n}} \frac{\Ptrans{\vtheta_n, -\vm_n}{\vtheta_*, -\vm_*}}{\Ptrans{\vtheta_*, \vm_*}{\vtheta_n, \vm_n}},
\label{eq:metropolis-hastings}
\end{equation}
otherwise we reject the sample, that is, set $\vs_{n+1}=\vs_{n}$. We negate the momentum so the integrator applied backwards would retrace the same steps.\footnote{Hamiltonian dynamics are time-reversible. We also need to use a time-reversible integrator, like the leapfrog or velocity Verlet, so the ratio of $g$ in \cref{eq:metropolis-hastings} is 1.}
 Afterwards, we are free to negate the momentum again without a rejection step, since $K(-\vm) = K(\vm)$, so it leaves the target distribution~(\ref{eq:boltzmann}) invariant \citep[eq.~5.20]{neal2012handbook}. For regular HMC, this can be ignored because the momentum gets resampled immediately.

\section{SGHMC acceptance probabilities are zero \label{sec:zero_probs}}

All the Markov chains we consider in this paper are discretisations of
Langevin's stochastic differential equations (SDEs) which, borrowing notation
from \citet{wenzel20posterior}, are
\begin{align}
  \dd\vtheta &= \vM^{{-1}} \vm \,\dd{t}
  \label{eq:langevin-parameter} \\
  \dd\vm &= -\gamma \vm \,\dd{t} - \gradat{}\,\dd{t} + \vM^{1/2}\sqrt{2\gamma T}\,\dd{\vW} .
  \label{eq:langevin-momentum}
\end{align}
Here, $\gamma \in \mathbb{R}^{+}$ is a friction coefficient and $\vW$ is a standard
Wiener process, with covariance $\mI$.

If $\gamma = 0$, these become Hamilton's equations, used for HMC.
Setting $\dd\vm=0$, we obtain the overdamped Langevin equation, which discretises
to the Metropolis-adjusted Langevin algorithm (MALA; e.g., \cite{roberts1998mala}).

For time step $h$, the symplectic Euler-Maruyama scheme in \citet[SGHMC]{sghmc} and \citet{wenzel20posterior} simulates these SDEs as follows. Using noise $\vepsilon \sim \Normal{{\boldsymbol 0}, \mI}$,
\begin{align}
  \vm_{n+1} &= (1 - h\gamma) \vm_{n} - h\gradat{_{n}} + \sqrt{h}\mM^{1/2}\sqrt{2\gamma T} \vepsilon
                  \label{eq:symplectic-euler-momentum} \\
  \vtheta_{n+1} &= \vtheta_{n} + h\mM^{{-1}} \vm_{n+1} \;.
                  \label{eq:symplectic-euler-parameter}
\end{align}
Like in HMC, it is the discretisation, not the SDE, that determines our simulation's Markov transition kernel $g$. Unfortunately, we will see that the backward transition is unrealizable, so its density $\Ptrans{\vtheta_{n}, -\vm_{n}}{\vtheta_{n+1}, -\vm_{n+1}}$ is almost always zero. Because this density appears in the numerator of \cref{eq:metropolis-hastings}, the resulting acceptance probability for the Euler-Maruyama scheme is zero as well.

\begin{theorem}
  For any starting point $(\vtheta_{n},\vm_{n})$ and step size $h > 0$, we sample
  $(\vtheta_{n+1},\vm_{n+1})$ using
  equations~(\ref{eq:symplectic-euler-momentum},~\ref{eq:symplectic-euler-parameter}).
  Then, the backward transition density
  $\Ptrans{\vtheta_{n}, -\vm_{n}}{\vtheta_{n+1}, -\vm_{n+1}}$ is zero, with probability 1 over the randomness of the forward transition.
  \label{thm:zero-density}
\end{theorem}
\begin{proof}
Consider the transition distribution $\Ptrans{\vtheta_*, \vm_*}{\vtheta, \vm}$. Because it has been generated from equations~(\ref{eq:symplectic-euler-momentum},~\ref{eq:symplectic-euler-parameter}), any realization from $g$ must satisfy them. In particular, it has to satisfy \cref{eq:symplectic-euler-parameter}, that is, $\vtheta_{*} = \vtheta + h\mM^{-1}\vm_{*}$. Thus,
 for the purposes of the acceptance probability (eq.~\ref{eq:metropolis-hastings}), we can define the density of $g$ as zero outside of its support,
  \begin{equation}
    \Ptrans{\vtheta_{*}, \vm_{*}}{\vtheta,\vm} = \begin{cases}
      \Normal{\vm_{*} \mvbar (1 - h\gamma)\vm - h\gradat{}, h\mM2\gamma T}
      &\text{if } \vtheta_{*} = \vtheta + h\mM^{-1}\vm_{*} \\
     0 & \text{otherwise}.\end{cases}
   \label{eq:transition-density}
 \end{equation}

  Now, let us check whether this condition can simultaneously hold for the  forward and backward transitions in \cref{eq:metropolis-hastings}. That is, is it possible that
  \begin{equation}
      \vtheta_{n+1} = \vtheta_n + h\mM^{-1}\vm_{n+1}
      \quad\text{and}\quad 
      \vtheta_{n} = \vtheta_{n+1} + h\mM^{-1}(-\vm_{n}) \; \text{?}
      \label{eq:forward-and-backward-conditions}
  \end{equation}
  Solving the linear system we can see that, since $\mM$ is positive definite
  and $h>0$, \cref{eq:forward-and-backward-conditions} is true if and only if
  $\vm_n = \vm_{n+1}$. Since the step size is $h > 0$, this only happens when
  the Gaussian random draw $\vepsilon$ in \cref{eq:symplectic-euler-momentum} is
  equal to
  $\vv = \bra{h\mM 2\gamma T}^{-1/2} \bra{ h\gamma \vm_n + h\gradat{_n}}$.
  Because the mean of $\vepsilon$ is zero (and thus $\ne \vv$) and its
  distribution is Gaussian, this is a probability zero event.\footnote{This also holds for any non-singular distributions on $\vepsilon$. It is also undesirable, because the momentum is supposed to change over time when following Langevin dynamics.}

Thus, with probability 1, $\vepsilon \ne \vv$, which implies $\vm_n \ne \vm_{n+1}$. We know the left-hand-side expression of \cref{eq:forward-and-backward-conditions} is true, so that implies $\vtheta_n \ne \vtheta_{n+1} + h\mM^{-1}(-\vm_n)$, which implies by \cref{eq:transition-density} that $\Ptrans{\vtheta_{n}, -\vm_{n}}{\vtheta_{n+1}, -\vm_{n+1}} = 0$.  
\end{proof}


The easiest way to ensure $\vm_{n} = \vm_{n+1}$, such that the acceptance probability
remains nonzero, is to set the time step $h=0$. This, however, makes the sampler
useless. If we take the limit $h \to 0$ instead, the Euler-Maruyama scheme gets
arbitrarily close to the true SDE trajectory. Crucially however, per \cref{thm:zero-density}, the acceptance probability
remains $0$ for any $h>0$. Thus, it is impossible to use the acceptance probability to monitor the discretisation error. The Euler-Maruyama scheme \emph{cannot} satisfy detailed balance. So far, we have not considered stochastic gradients, but this result includes them: it still holds if we substitute an arbitrary $\vg_{n}$ for $\gradat{_n}$.


\section{GGMC solves the problem by using a different integrator
\label{sec:GGMC}}

To get a nonzero acceptance probability, we can just use a backward-realizable integrator. The one we present here has been discovered several times in statistics \citep{horowitz1991, neal1993report} and molecular dynamics \citep{bussi-parrinello,lm-obabo}. We call it Gradient-guided Monte Carlo (GGMC), following \citet{horowitz1991}, because it operates similarly to random walk Metropolis-Hastings on the momentum, but with a parameter gradient influencing the random walk's direction.
\citet[sec.~5.1]{neal1993report} calls it the ``stochastic dynamics'' method, or ``Langevin Monte Carlo (LMC) with partial momentum refreshment'' \citep{neal2012handbook}.
\citet{lm-obabo} call this integrator ``OBABO'' because of the order of its steps.

The integrator splits Langevin dynamics
(eqs.~\ref{eq:langevin-momentum}--\ref{eq:langevin-parameter}) into three parts,
which can be solved exactly, and applies them in sequence \citep{lm-obabo}. Let $a = e^{-\gamma h}$, and $\vepsilon,\vepsilon' \sim \Normal{{\boldsymbol 0}, \mI}$ be Gaussian random variables with identity covariance. Then
\begin{align}
  \tag{O.1} \label{eq:O.1}
  \vm_{n+{1/4}} &= \sqrt{a} \vm_{n} + \sqrt{(1 - a) T} \mM^{1/2}\vepsilon, \\
  \tag{B.1}\label{eq:B.1}
  \vm_{n+{1/2}} &= \vm_{n+{1/4}} - \frac{h}{2}\gradat{_{n}}, \\
  \tag{A}\label{eq:A}
  \vtheta_{n+1} &= \vtheta_{n} + h\mM^{{-1}}\vm_{n+{1/2}}, \\
  \tag{B.2}\label{eq:B.2}
  \vm_{n+{3/4}} &= \vm_{n+{1/2}} - \frac{h}{2}\gradat{_{n+1}}, \\
  \tag{O.2}\label{eq:O.2}
  \vm_{n+1} &= \sqrt{a} \vm_{n+{3/4}} + \sqrt{(1 - a) T} \mM^{{1/2}} \vepsilon'.
\end{align}
Steps~\ref{eq:O.1}~and~\ref{eq:O.2} are partial momentum refreshment, which is exact; and steps \ref{eq:B.1}, \ref{eq:A} and \ref{eq:B.2} form one leapfrog step. Thus, the discretisation error is locally $O(h^3)$ \citep{neal2012handbook}. Establishing the trajectory discretisation error is complicated, because it requires solving a partial differential equation (PDE). However, all found PDE solutions indicate that the error is $O(h^{2})$ \citep{lm-obabo,scc-time-rescaling}, the same as HMC. 

This scheme generalizes several popular samplers.
To obtain HMC, we may repeat BAB many times, and use O as momentum refreshment. Alternatively,
setting $\gamma = 0$ implies $a=1$, so the O steps disappear, and the integrator becomes one leapfrog step. Adding a separate momentum refreshment then also yields HMC. To obtain MALA\footnote{\citet{lm-obabo} recommend the BAOAB ordering for more accurate sampling in MALA.} or SGLD \citep{sgld}, we take $\gamma \to \infty$. This implies $a \to 0$, which means that step~\ref{eq:O.1} entirely ignores the incoming momentum $\vm_n$ and just samples a fresh one. Because of this, we do not need to calculate the next momentum $\vm_{n+1}$ at the end of a step, or represent the momentum at all. Thus, steps \ref{eq:B.2} and \ref{eq:O.2} fall out. We may also fold steps~\ref{eq:O.1} and \ref{eq:B.1} into \ref{eq:A} to obtain the more common single-step expression for SGLD.

Note that there are other modern methods that simulate Langevin dynamics while preserving reversibility. For instance, \citet{nogin} propose a variant of ABOBA that corrects for the (known or estimated) covariance of the gradient noise, similar to SGHMC, to get the correct $O(h^2)$ statistics. Also, \citet{amagold} develop an integrator that is asymmetric, but nevertheless realizable in the reverse direction, and thus has a nonzero acceptance probability.

\paragraph{Acceptance probability.}
The key ingredient is the ratio of backward and
forward transition probabilities, which we derive in \cref{apd:obabo-accept}.
The resulting acceptance probability is
\begin{equation}
 \log \mhaccept{\vtheta_{n+1}, \vm_{n+1}}{\vtheta_{n}, \vm_{n}} = -\frac{1}{T}\bra{U(\vtheta_{n+1}) - U(\vtheta_{n}) + K(\vm_{n+3/4}) - K(\vm_{n+1/4})}.
\label{eq:ggmc-accept}
\end{equation}
This is the same expression as in HMC or Langevin MC, were we to regard only the inner
Leapfrog step as the transition to accept or reject.
Unlike Hamiltonian dynamics ($\gamma = 0$), underdamped Langevin dynamics ($0 < \gamma < \infty$) are not in general time-reversible, even when simulated exactly \citep{sghmc-complete-recipe}.\footnote{They are still backwards-\emph{realizable}, that is, the backwards trajectory is in the support of the transition kernel.} However, applying a
Metropolis-Hastings correction step forces the resulting Markov chain to be reversible.
It is interesting, then, that the acceptance probability depends only on the
accuracy of the reversible steps~BAB, and not
at all on the friction $\gamma$, which controls how much the momentum is refreshed and thus how irreversible the Markov chain is.

\paragraph{Empirical properties.} In Fig.~\ref{fig:cifar10}, we compare GGMC and HMC for different learning rates. GGMC has similar potential and kinetic energy for all chains. HMC has higher acceptance probability than GGMC (likely because the underlying chain in GGMC is irreversible). The behaviour of HMC varies from chain to chain, suggesting that it is unstable for larger $\ell$.
The only learning rate for which the acceptance probability is close to 1 is $10^{-6}$, which is much lower than is used in typical practice (see Fig.~\ref{fig:cifar10-full} in the appendix).

Overall, the qualitative behaviour of GGMC seems similar to SGHMC, but with the advantage that its accuracy can be monitored or corrected using the M-H acceptance probability. See \cref{apd:experiments} for more experiments and details.

\begin{figure}[ht]
  \centering
  \input{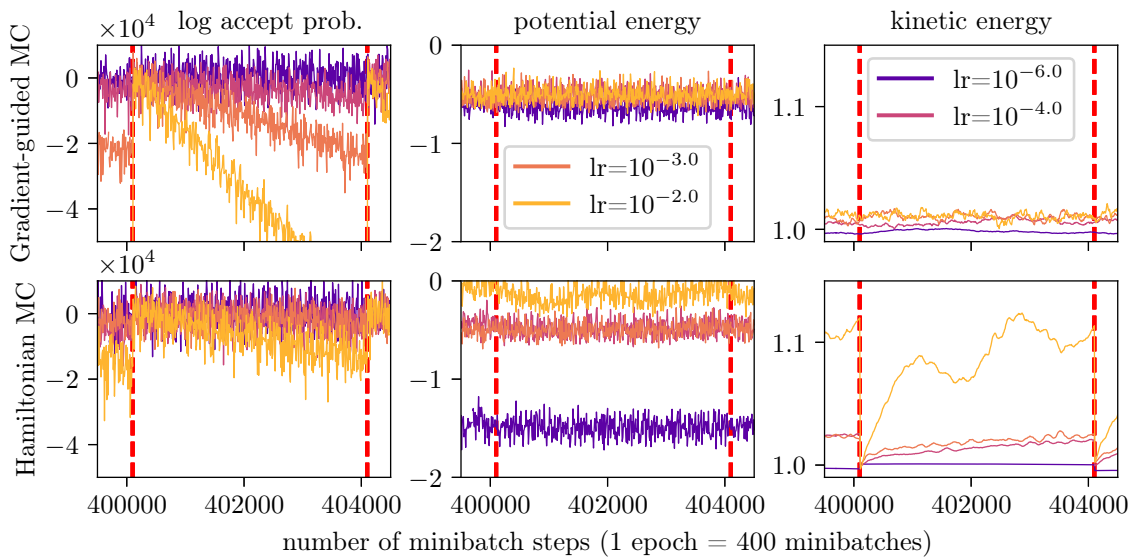}
  \vspace{-1.2ex}
  \caption{Comparison of different learning rates ($\ell\sim h$, see \cref{apd:parameter-relationship}),
    for GGMC and HMC applied to a ResNet-20 in the CIFAR-10 data set; for epochs numbered 1000--1010. Details and full training run in \cref{apd:experiments}. \label{fig:cifar10}}
\end{figure}

\section{Acceptance probabilities for multiple and stochastic steps}\label{sec:multi_step}
Already \citet{duane-hmc} observed that the Hamiltonian used to specify the transition probability $g$ need not be the one used to calculate the acceptance probability. Only the latter will determine what distribution the sampler converges to. That is, the sampler (HMC, or GGMC) may use stochastic or incorrect gradients, so long as the MH acceptance probability is calculated with respect to the target $\Ptarg{\vtheta}^{1/T}e^{-K(\vm)/T}$.

Unfortunately, calculating the exact acceptance probability after every GGMC step defeats the point of stochastic gradients. Instead, we can defer the acceptance probability calculation by many steps. This allows us to, for example, follow 10 epochs of stochastic gradients, and then spend one epoch calculating the acceptance probability. The resulting scheme multiplies the acceptance probabilities of each step together \citep{amagold}. Starting at $\vs_n$, after $N$ steps of MCMC, we can accept the resulting sample $\vs_{n+N}$ with probability $\min\bra{1, \mhaccept{\vs_{n+N}}{\vs_n}}$, where
\begin{equation}
\mhaccept{\vs_{n+N}}{\vs_n} = \prod_{i=1}^N \mhaccept{\vs_{n+i}}{\vs_{n+i-1}} = \frac{\Pboltzmann{\vs_{n+N}}}{\Pboltzmann{\vs_{n}}} \prod_{i=1}^N \frac{\Ptrans{\vs_{n+i}}{\vs_{n+i-1}}}{\Ptrans{\vs_{n+i-1}}{\vs_{n+i}}}.
\label{eq:multi-step-mh}
\end{equation}
Due to telescopic cancellation, we need only evaluate the exact likelihood at the end of many steps. This can be combined with other algorithms \citep[sec.~5.4]{neal2012handbook,sohl-dickstein-hmc} to reduce the rejection probability.
\citet[Lemma~2]{amagold} propose this deferred MH scheme and show its correctness for any sampler with a homogeneous Markov chain.
We extend it to symmetrically inhomogeneous Markov chains in \cref{apd:multi_step}. Our extension gives zero acceptance probability to asymmetric chains like the popular cosine learning rate schedule \citep{cosine-schedule}.

\section{Conclusion}

We have shown that SGHMC and similar popular samplers cannot be corrected using rejection, since their acceptance probabilities are zero.
We have then revisited GGMC, which solves this problem and generalizes SGLD and HMC.
Finally, we have also presented a way to compute the acceptance probabilities of GGMC across several steps, leading to a method for drawing exact posterior samples using only stochastic gradients.

It is time for the Bayesian deep learning community to let go of the Euler-Maruyama scheme. Instead, we should explore the breadth of integrators for Langevin dynamics which we have inherited from molecular dynamics and statistics. We have shown that correctness is achievable. Thus, we should  demand correctness from our samplers.

\acks{The authors would like to thank Andrew Foong, Samuel Power, and Austin Tripp for helpful discussions and comments on an earlier draft. AGA was supported by a UK Engineering and Physical Sciences Research Council studentship [1950008]. VF was supported by a PhD fellowship from the Swiss Data Science Center.}

\bibliography{references}

\begin{thebibliography}{25}
\providecommand{\natexlab}[1]{#1}
\providecommand{\url}[1]{\texttt{#1}}
\expandafter\ifx\csname urlstyle\endcsname\relax
  \providecommand{\doi}[1]{doi: #1}\else
  \providecommand{\doi}{doi: \begingroup \urlstyle{rm}\Url}\fi

\bibitem[Bussi and Parrinello(2007)]{bussi-parrinello}
Giovanni Bussi and Michele Parrinello.
\newblock Accurate sampling using {L}angevin dynamics.
\newblock \emph{Physical Review E}, 75, 2007.
\newblock \doi{10.1103/PhysRevE.75.056707}.

\bibitem[Chen et~al.(2014)Chen, Fox, and Guestrin]{sghmc}
Tianqi Chen, Emily Fox, and Carlos Guestrin.
\newblock Stochastic gradient {H}amiltonian {M}onte {C}arlo.
\newblock In \emph{Proceedings of the 31st International Conference on Machine
  Learning (ICML 2014)}, volume~32 of \emph{Proceedings of Machine Learning
  Research}, pages 1683--1691, 2014.
\newblock URL \url{http://proceedings.mlr.press/v32/cheni14.html}.

\bibitem[Coelho(2017)]{jug}
Luis~Pedro Coelho.
\newblock {J}ug: Software for parallel reproducible computation in python.
\newblock \emph{Journal of Open Research Software.}, 5:\penalty0 30, 2017.
\newblock \doi{10.5334/jors.161}.

\bibitem[Ding et~al.(2014)Ding, Fang, Babbush, Chen, Skeel, and Neven]{sgnht}
Nan Ding, Youhan Fang, Ryan Babbush, Changyou Chen, Robert~D Skeel, and Hartmut
  Neven.
\newblock Bayesian sampling using stochastic gradient thermostats.
\newblock In \emph{Advances in Neural Information Processing Systems},
  volume~27, pages 3203--3211, 2014.
\newblock URL
  \url{https://proceedings.neurips.cc/paper/2014/file/21fe5b8ba755eeaece7a450849876228-Paper.pdf}.

\bibitem[Duane et~al.(1987)Duane, Kennedy, Pendleton, and Roweth]{duane-hmc}
Simon Duane, A.D. Kennedy, Brian~J. Pendleton, and Duncan Roweth.
\newblock Hybrid {M}onte {C}arlo.
\newblock \emph{Physics Letters B}, 195\penalty0 (2):\penalty0 216 -- 222,
  1987.
\newblock \doi{10.1016/0370-2693(87)91197-X}.

\bibitem[Geyer(1992)]{geyer1992practical}
Charles~J Geyer.
\newblock Practical {M}arkov chain {M}onte {C}arlo.
\newblock \emph{Statistical science}, pages 473--483, 1992.

\bibitem[Hastings(1970)]{hastings1970monte}
W.K. Hastings.
\newblock {M}onte {C}arlo sampling methods using {M}arkov chains and their
  applications.
\newblock \emph{Biometrika}, 57\penalty0 (1):\penalty0 97--109, 1970.

\bibitem[Horowitz(1991)]{horowitz1991}
Alan~M. Horowitz.
\newblock A generalized guided {M}onte {C}arlo algorithm.
\newblock \emph{Physics Letters B}, 268\penalty0 (2):\penalty0 247 -- 252,
  1991.
\newblock ISSN 0370-2693.
\newblock \doi{https://doi.org/10.1016/0370-2693(91)90812-5}.

\bibitem[Leimkuhler and Matthews(2015)]{lm-obabo}
Ben Leimkuhler and Charles Matthews.
\newblock \emph{Numerical Methods for Stochastic Molecular Dynamics},
  chapter~7, pages 261--328.
\newblock Springer International Publishing, Cham, 2015.
\newblock ISBN 978-3-319-16375-8.
\newblock \doi{10.1007/978-3-319-16375-8_7}.

\bibitem[Leimkuhler et~al.(2019)Leimkuhler, Matthews, and Vlaar]{adlala}
Benedict Leimkuhler, Charles Matthews, and Tiffany Vlaar.
\newblock Partitioned integrators for thermodynamic parameterization of neural
  networks.
\newblock 2019.
\newblock URL \url{http://arxiv.org/abs/1908.11843v2}.

\bibitem[Ma et~al.(2019)Ma, Fox, and Chen]{sghmc-complete-recipe}
Yi-An Ma, Emily Fox, and Tianqi Chen.
\newblock Irreversible samplers from jump and continuous markov processes.
\newblock \emph{Statistics and Computing}, 29:\penalty0 177--202, 2019.
\newblock URL \url{https://arxiv.org/abs/1608.05973}.

\bibitem[Matthews and Weare(2018)]{nogin}
Charles Matthews and Jonathan Weare.
\newblock {L}angevin {M}arkov chain {M}onte {C}arlo with stochastic gradients.
\newblock 2018.
\newblock URL \url{http://arxiv.org/abs/1805.08863v2}.

\bibitem[Metropolis et~al.(1953)Metropolis, Rosenbluth, Rosenbluth, Teller, and
  Teller]{metropolis1953equation}
Nicholas Metropolis, Arianna~W Rosenbluth, Marshall~N Rosenbluth, Augusta~H
  Teller, and Edward Teller.
\newblock Equation of state calculations by fast computing machines.
\newblock \emph{The Journal of Chemical Physics}, 21\penalty0 (6):\penalty0
  1087--1092, 1953.

\bibitem[Neal(1993)]{neal1993report}
Radford~M. Neal.
\newblock Probabilistic inference using {M}arkov chain {M}onte {C}arlo methods.
\newblock Technical Report CRG-TR-93-1, 1993.
\newblock URL \url{https://www.cs.toronto.edu/~radford/review.abstract.html}.

\bibitem[Neal(1996)]{neal1996bayesian}
Radford~M. Neal.
\newblock \emph{Bayesian learning for neural networks}, volume 118.
\newblock Springer, 1996.

\bibitem[Neal(2012)]{neal2012handbook}
Radford~M. Neal.
\newblock {MCMC} using {H}amiltonian dynamics.
\newblock 2012.
\newblock URL \url{http://arxiv.org/abs/1206.1901v1}.

\bibitem[Paszke et~al.(2019)Paszke, Gross, Massa, Lerer, Bradbury, Chanan,
  Killeen, Lin, Gimelshein, Antiga, Desmaison, Kopf, Yang, DeVito, Raison,
  Tejani, Chilamkurthy, Steiner, Fang, Bai, and Chintala]{pytorch}
Adam Paszke, Sam Gross, Francisco Massa, Adam Lerer, James Bradbury, Gregory
  Chanan, Trevor Killeen, Zeming Lin, Natalia Gimelshein, Luca Antiga, Alban
  Desmaison, Andreas Kopf, Edward Yang, Zachary DeVito, Martin Raison, Alykhan
  Tejani, Sasank Chilamkurthy, Benoit Steiner, Lu~Fang, Junjie Bai, and Soumith
  Chintala.
\newblock Pytorch: An imperative style, high-performance deep learning library.
\newblock In \emph{Advances in Neural Information Processing Systems 32
  (NeurIPS 2019)}, pages 8024--8035. 2019.

\bibitem[Roberts and Rosenthal(1998)]{roberts1998mala}
Gareth~O. Roberts and Jeffrey~S. Rosenthal.
\newblock Optimal scaling of discrete approximations to {L}angevin diffusions.
\newblock \emph{Journal of the Royal Statistical Society: Series B (Statistical
  Methodology)}, 60\penalty0 (1):\penalty0 255--268, 1998.
\newblock \doi{https://doi.org/10.1111/1467-9868.00123}.
\newblock URL
  \url{https://rss.onlinelibrary.wiley.com/doi/abs/10.1111/1467-9868.00123}.

\bibitem[Sivak et~al.(2014)Sivak, Chodera, and Crooks]{scc-time-rescaling}
David~A. Sivak, John~D. Chodera, and Gavin~E. Crooks.
\newblock Time step rescaling recovers continuous-time dynamical properties for
  discrete-time {L}angevin integration of nonequilibrium systems.
\newblock \emph{The Journal of Physical Chemistry B}, 118, 2014.
\newblock \doi{10.1021/jp411770f}.

\bibitem[Sohl-Dickstein et~al.(2014)Sohl-Dickstein, Mudigonda, and
  DeWeese]{sohl-dickstein-hmc}
Jascha Sohl-Dickstein, Mayur Mudigonda, and Michael DeWeese.
\newblock Hamiltonian monte carlo without detailed balance.
\newblock In \emph{Proceedings of the 31st International Conference on Machine
  Learning (ICML 2014)}, volume~32 of \emph{Proceedings of Machine Learning
  Research}, pages 719--726, 2014.
\newblock URL \url{http://proceedings.mlr.press/v32/sohl-dickstein14.html}.

\bibitem[Welling and Teh(2011)]{sgld}
Max Welling and Yee~Whye Teh.
\newblock {B}ayesian learning via stochastic gradient {L}angevin dynamics.
\newblock In \emph{Proceedings of the 28th International Conference on Machine
  Learning (ICML 2011)}, pages 681--688, 2011.

\bibitem[Wenzel et~al.(2020)Wenzel, Roth, Veeling, {\'S}wi{\k{a}}tkowski, Tran,
  Mandt, Snoek, Salimans, Jenatton, and Nowozin]{wenzel20posterior}
Florian Wenzel, Kevin Roth, Bastiaan~S. Veeling, Jakub {\'S}wi{\k{a}}tkowski,
  Linh Tran, Stephan Mandt, Jasper Snoek, Tim Salimans, Rodolphe Jenatton, and
  Sebastian Nowozin.
\newblock How good is the {B}ayes posterior in deep neural networks really?,
  2020.
\newblock URL \url{http://arxiv.org/abs/2002.02405v1}.

\bibitem[Wu et~al.(2020)Wu, K{\"o}hler, and No{\'e}]{wu20normalizing}
Hao Wu, Jonas K{\"o}hler, and Frank No{\'e}.
\newblock Stochastic normalizing flows.
\newblock 2020.
\newblock URL \url{http://arxiv.org/abs/2002.06707v3}.

\bibitem[Zhang et~al.(2019)Zhang, Li, Zhang, Chen, and Wilson]{cosine-schedule}
Ruqi Zhang, Chunyuan Li, Jianyi Zhang, Changyou Chen, and Andrew~Gordon Wilson.
\newblock Cyclical stochastic gradient {MCMC} for {B}ayesian deep learning.
\newblock 2019.
\newblock URL \url{http://arxiv.org/abs/1902.03932v2}.

\bibitem[Zhang et~al.(2020)Zhang, Cooper, and Sa]{amagold}
Ruqi Zhang, A.~Feder Cooper, and Christopher~De Sa.
\newblock {AMAGOLD}: Amortized {M}etropolis adjustment for efficient stochastic
  gradient {MCMC}.
\newblock 2020.
\newblock URL \url{http://arxiv.org/abs/2003.00193v1}.

\end{thebibliography}

\newpage
\appendix

\counterwithin{figure}{section}
\counterwithin{table}{section}

\section{Acceptance probability calculation for GGMC}\label{apd:obabo-accept}
The procedure used by \citet{bussi-parrinello} is to start with the joint
distribution of the noises $\vepsilon,\vepsilon'$ and transform it to the
distribution over $\vtheta_{n+1},\vm_{n+1} \vbar \vtheta_{n},\vm_{n}$ using the
density transformation formula with the determinant of the Jacobian.

However, in this case, there is somewhat more illustrative way. Consider
\begin{equation}
  \vr = \sqrt{(1-a)T} \mM^{-1/2} \vepsilon,\quad\text{ and thus }\quad \vr\sim \Normal{{\boldsymbol 0}, (1-a)T\mM},
\end{equation}
and similarly for $\vr'$ and $\vepsilon'$. Note that $\vr = \vm_{n+1/4} - \sqrt{a}\vm_{n}$, and $\vr' = \vm_{n+1} - \sqrt{a}\vm_{n+3/4}$.
By substituting $\vr,\vr'$ into \cref{eq:B.1,eq:B.2}, we can write $\vtheta_{n+1},\vm_{n+1}$ as
\begin{align}
  \vtheta_{n+1} &= \vtheta_n + h\mM^{-1} \bra{\sqrt{a} \vm_{n} - \frac{h}{2}\gradat{_{n}} + \vr}, \\
\vm_{n+1} &=
  a \vm_{n} - \frac{\sqrt{a}h}{2}\bra{\gradat{_{n}} + \gradat{_{n+1}}} + \sqrt{a}\vr + \vr'.
\end{align}

The reciprocal of the determinant of the Jacobian is thus
\begin{equation}
\det{\begin{matrix}
{\dd\vtheta_{n+1}}/{\dd\vr} & {\dd\vtheta_{n+1}}/{\dd\vr'} \\
{\dd\vm_{n+1}}/{\dd\vr} & {\dd\vm_{n+1}}/{\dd\vr'}
\end{matrix}}^{-1}
= \det{\begin{matrix} h\vM^{-1} & 0 \\ \sqrt{a} & \mI \end{matrix}}^{{-1}} = h^{-1}\det{\vM},
\end{equation}
which is a constant, so it will be eliminated when dividing by the backward transition density.

Putting this all together, the forward transition density is
\begin{equation}
  \begin{aligned}
    \Ptrans{\vtheta_{n+1},\vm_{n+1}}{\vtheta_{n},\vm_{n}} = h^{-1}\det{\vM} \,
    &\Normal{\vm_{n+1/4} - \sqrt{a}\vm_{n} \mvbar {\boldsymbol 0}, (1-a)T\mM}\\
    &\Normal{\vm_{n+1} - \sqrt{a}\vm_{n+3/4} \mvbar {\boldsymbol 0}, (1-a)T\mM}.
  \end{aligned}
\end{equation}

By substituting  $\vtheta_{n+1}\leftrightarrow \vtheta_{n}$,
$\vm_{n}\rightarrow -\vm_{n+1}$,
$\vm_{n+1/4}\rightarrow -\vm_{n+3/4}$,
$\vm_{n+3/4}\rightarrow -\vm_{n+1/4}$,
$\vm_{n+1}\rightarrow -\vm_{n}$,
we can obtain the backward transition probability
\begin{equation}
  \begin{aligned}
    \Ptrans{\vtheta_{n},-\vm_{n}}{\vtheta_{n+1},-\vm_{n+1}} = h^{-1}\det{\vM} \,
    &\Normal{-\vm_{n+3/4} + \sqrt{a}\vm_{n+1} \mvbar {\boldsymbol 0}, (1-a)T\mM} \\
    &\Normal{-\vm_{n} + \sqrt{a}\vm_{n+1/4} \mvbar {\boldsymbol 0}, (1-a)T\mM}.
  \end{aligned}
\end{equation}

Dividing them, the constant terms from the Gaussians and $h^{-1}\det{\vM}$ disappear, and we obtain
\begin{equation}
  \begin{aligned}
\log  & \frac{\Ptrans{\vtheta_{n},-\vm_{n}}{\vtheta_{n+1},-\vm_{n+1}}}{\Ptrans{\vtheta_{n+1},\vm_{n+1}}{\vtheta_{n},\vm_{n}}}
= \\
& -\frac{1}{(1-a) T} \bigg(
 K\bra{-\vm_{n+3/4} + \sqrt{a}\vm_{n+1}} + K\bra{-\vm_{n} + \sqrt{a}\vm_{n+1/4}} \\
& \hspace{2.5cm} - K\bra{ \vm_{n+1/4} - \sqrt{a}\vm_{n}  } - K\bra{ \vm_{n+1} - \sqrt{a}\vm_{n+3/4}  }
\bigg)
  \end{aligned}
\end{equation}
where $K\bra{\vv} = \frac{1}{2}\vv^{\tp}\mM^{{-1}}\vv$ is the kinetic energy defined in \cref{sec:zero_probs}. This simplifies greatly, to
\begin{equation}
  \log  \frac{\Ptrans{\vtheta_{n},-\vm_{n}}{\vtheta_{n+1},-\vm_{n+1}}}{\Ptrans{\vtheta_{n+1},\vm_{n+1}}{\vtheta_{n},\vm_{n}}}
  =-\frac{1}{T} \bra{ K\bra{\vm_{n}} - K\bra{\vm_{n+1}} + K\bra{\vm_{n+3/4}} - K\bra{\vm_{n+1/4}} }.
  \label{eq:forward-backward-log-ratio}
\end{equation}
The log ratio of the Boltzmann densities (\cref{eq:boltzmann}) is
\begin{equation}
  \begin{aligned}
  \log \frac{\Pboltzmann{\vtheta_{n+1}, \vm_{n+1}}}{\Pboltzmann{\vtheta_{n}, \vm_{n}}} &= -\frac{1}{T}\bra{U(\vtheta_{n+1}) - U(\vtheta_{n}) + K(\vm_{n+1}) - K(\vm_{n})}.
  \label{eq:boltzmann-log-ratio}
  \end{aligned}
\end{equation}
Now, to obtain the acceptance probability in \cref{eq:metropolis-hastings} we add eqs.~(\ref{eq:forward-backward-log-ratio})~and~(\ref{eq:boltzmann-log-ratio}), obtaining
\begin{equation}
  \tag{\ref{eq:ggmc-accept}}
  \log \mhaccept{\vtheta_{n+1}, \vm_{n+1}}{\vtheta_{n}, \vm_{n}} = -\frac{1}{T}\bra{U(\vtheta_{n+1}) - U(\vtheta_{n}) + K(\vm_{n+3/4}) - K(\vm_{n+1/4})}
\end{equation}
which is exactly the same as \cref{eq:ggmc-accept}.

\section{Relationship between GGMC and Gradient Descent parameters}\label{apd:parameter-relationship}

Following \citep{wenzel20posterior}, we found it useful to express the
parameters of the GGMC sampler (momentum permanence $a$ and time step $h$) in terms of the parameters used most commonly in gradient descent deep learning (learning rate $\ell$ and momentum $\beta$). This way, we can use parameters that have been found to work well by the deep learning community.

Gradient descent first updates the momentum, and then updates the parameters; in a time-asymmetric way, unlike GGMC. Thus, to obtain the parameter equivalences, we must start from the \ref{eq:O.2}~step, combine it with the \ref{eq:O.1}~step of the next iteration, and end at the \ref{eq:B.2} step. The equivalences then are as follows, where $N$ is the number of data points in the training set,
\begin{align}
  \beta &= a = e^{-\gamma h} &  \gamma &= - \sqrt{\frac{N}{\ell}} \log \beta  \\
  \ell &= N \, h^{2} & h &= \sqrt{\frac{\ell}{N}}. \\
\intertext{The equivalence between $\beta$ and $\gamma$ is different for the
  symplectic Euler discretization. In that case it is}
  \beta_{\text{SE}} &= 1 - \gamma h & \gamma &= (1 - \beta_{SE})\sqrt{\frac{N}{\ell}}.
\end{align}

\section{Multiple steps in symmetrically inhomogeneous Markov chains}\label{apd:multi_step}

We can extend the proof for the correctness of deferred acceptance probabilities \citep[Lemma~2]{amagold} to inhomogeneous Markov chains
by considering the following scheme. The scheme will force the inhomogeneity to be symmetric in time.

We augment the state of the Markov chain with $N$ random variables, $\{\zeta_i\}_{i=1}^N$, that will determine the randomness of the intermediate steps. Their target distributions have to be time-symmetric, i.e. $p_{\zeta_i}(\cdot) = p_{\zeta_{N-i+1}}(\cdot)$ for all $i$. Before an MCMC multi-step, we sample all the $\zeta_i$ from their target distribution. Then, we use them to provide all the randomness for the intermediate steps, so each step is just the result of a deterministic transformation:
\begin{equation}
    \vs_{n+i} = f_i(\vs_{n+i-1}; \zeta_1, \dots, \zeta_i),\text{ for all }i=1,\dots,N.
\end{equation}
The functions $f_i$ have to be time-symmetric as well ($f_i = f_{N-i+1}$). Otherwise, the backward transition is not realizable using the same $\zeta_i$ as the forward one, so the acceptance probability is zero. If they are time-symmetric, the resulting acceptance probability is \cref{eq:multi-step-mh}.

In this view, the cosine schedule \citep{cosine-schedule} has acceptance probability zero, even if a backward-realizable integrator is used. The cosine schedule is not a time-symmetric transformation: the schedule of learning rates is high at the beginning and low at the end of a cycle.

\citet{wu20normalizing} employ a similar construction to use MCMC samplers as normalizing flows.
\citet[Section~3.5.1]{neal1996bayesian} proposed the ``partial gradient'' construction: a variant of this for HMC and dataset minibatching in particular, motivated by the same operator-splitting view of integrators as the leapfrog.



\section{Experimental details}\label{apd:experiments}
We use the ResNet-20 architecture from \citep{wenzel20posterior} on the CIFAR-10 data set with data augmentation. We vary the learning rate as $\lbrace 10^{-6}, 10^{-4}, 10^{-3}, 10^{-2}, 10^{-1} \rbrace$, discarding the last one because of its unstable results. Training runs for 3000 epochs using stochastic gradients with minibatch size 125 (which is divisible by 50000, the number of data points). Every 10 epochs (a cycle) we evaluate the acceptance probability exactly and store a sample. \cref{fig:cifar10} shows one such cycle, and \cref{fig:cifar10-full} shows the whole training runs. We do not apply a Metropolis-Hastings correction, but merely examine it to determine the correctness of the sampler.

\begin{figure}[ht]
  \centering
  \input{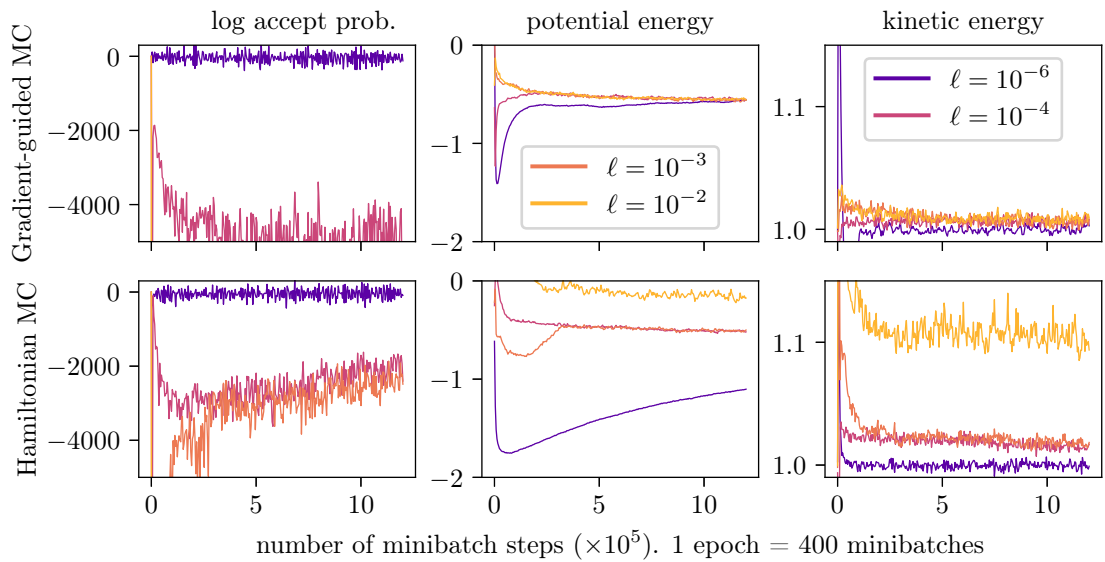}
  \vspace{-1.2ex}
  \caption{Comparison of different learning rates
    for GGMC and HMC applied to a ResNet-20 in the CIFAR-10 data set; for 3000 epochs. Samples shown are taken every 10 epochs. Quantities as displayed are averaged over the full CIFAR-10 data set, but a random data augmentation transformation is applied to each example. \label{fig:cifar10-full}}
\end{figure}

We implemented the experiments in PyTorch \citep{pytorch} for GPU acceleration and gradients, using Jug \citep{jug} for job scheduling.

\end{document}